\documentclass{article}
\usepackage{ijcai16}

\usepackage{times}

\usepackage{helvet}
\usepackage{courier}

\usepackage{url}
\usepackage{graphicx}
\usepackage{epstopdf}
\usepackage{caption}
\usepackage{subcaption}
\usepackage{hyperref}
\usepackage{algpseudocode}
\usepackage{amsmath}
\usepackage{amssymb}
\usepackage{amsthm}

\newtheorem{theorem}{Theorem}
\DeclareCaptionType{algorithm}
\renewcommand{\algorithmicrequire}{\textbf{Input:}}
\renewcommand{\algorithmicensure}{\textbf{Output:}}

\def\Tr{{\mbox{Tr}}}

\title{Regularized Singular Value Decomposition and \\ Application to Recommender System}
\author{Shuai Zheng, Chris Ding, Feiping Nie \\
University of Texas at Arlington\\
\texttt{zhengs123@gmail.com}, \texttt{chqding@uta.edu},\\ \texttt{feipingnie@gmail.com}}

\begin{document}

\maketitle

\begin{abstract}
Singular value decomposition (SVD) is the mathematical basis of principal component analysis (PCA). Together, SVD and PCA are one of the most widely used mathematical formalism/decomposition in machine learning, data mining, pattern recognition, artificial intelligence, computer vision, signal processing, etc..
In recent applications, regularization becomes an increasing trend.
In this paper, we present a regularized SVD (RSVD), present an efficient computational algorithm, and provide several theoretical analysis.
We show that although RSVD is non-convex, it has a closed-form global optimal solution.
Finally, we apply RSVD to the application of recommender system and experimental result show that RSVD outperforms SVD significantly.
\end{abstract}

\section{Introduction}
{\it Singular value decomposition} (SVD), its statistical form {\it principal component analysis} (PCA) and
{\it Karhunen-Loeve Transform} in signal processing,
are one of the most widely used mathematical formalism/decomposition in machine learning, data mining, pattern recognition, artificial intelligence, computer vision, signal processing, etc..

Mathematically, SVD can be seen as the best low-rank approximation to a rectangle matrix. The left and right singular vectors are mutually orthogonal, and provide orthogonal basis for row and column subspaces.
When the data matrix are centered as in most statistical analysis,
the singular vectors become eigenvectors of the covariance matrix and provide mutually uncorrelated/de-correlated subspaces which are much easier to use for statistical analysis. This form of SVD is generally referred to as PCA, and is widely used in statistics.

In its most simple form, SVD/PCA provides the most widely used dimension reduction for pattern analysis and data mining. SVD/PCA has numerous applications in engineering, biology, and social science \cite{jolliffe2005principal,zou2006sparse,zheng2014kernel,zheng2015closed,zheng2016harmonic,zheng2017machine}, such as handwritten zip code classification \cite{friedman2001elements}, human face recognition \cite{hancock1996face}, gene expression data analysis \cite{alter2000singular}, recommender system \cite{billsus1998learning}. As many big data, deep learning, cloud computing technologies were developed \cite{zheng2011analysis,zhang2012virtual,williams2014tidewatch,zheng2016accelerating,zheng2017long}, SVD/matrix decomposition has been integrated into commercial big data platforms, such as Hadoop Mahout framework.

In recent developments of machine learning and data mining, regularization becomes an increasing trend. Adding a regularization term to the loss function can increase the smoothness of the factor matrices and introduce more zero components to the factor matrices, such as sparse PCA \cite{shen2008sparse} \cite{guan2009sparse}. Sparse PCA has many applications in text mining, finance and gene data analysis \cite{zhang2012sparse} \cite{d2007direct}. Minimal Support Vector Machine \cite{zheng2018minimal} enforces sparsity on the number of support vectors.
In this paper, we present a regularized SVD (RSVD), present an efficient computational algorithm, and provide several theoretical analysis.
We show that although the RSVD is a non-convex formulation, it has a global optimal closed-form solution. Finally, we apply RSVD to recommender system on four real life datasets. RSVD based recommender system outperforms the standard SVD based recommender system.

{\bf Notations.} In this paper, matrices are written in uppercase letters, such as $X, \; Y$.
$\Tr(X)$ denotes the trace operation for matrix $X$.

\section{Regularized SVD (RSVD)}
Assume there is a matrix $X \in \Re^{n \times m}$. Regularized SVD (RSVD) tries to find low-rank approximation using regularized factor matrices $U$ and $V$. The objective function is proposed as
\begin{align}
\label{eq:j1}
J_1 = \| X - UV^T \|^2_F + \lambda \|U\|^2_F + \lambda \|V\|^2_F,
\end{align}%
where low-rank regularized factor matrices $U \in \Re^{n \times k}$ and $V \in \Re^{m \times k}$, $k$ is the rank of regularized SVD. Minimizing Eq.(\ref{eq:j1}) is a multi-variable problem. We will now present a faster Algorithm \ref{alg:rsvd} to solve this problem.

Eq.(\ref{eq:j1}) can be minimized in 2 steps:

A1. Fixing $V$, solve $U$. Take derivative of Eq.(\ref{eq:j1}) with respect to $U$ and set it to zero,
\begin{align}
\frac{\partial J_1}{\partial U} = -XV + UV^TV + \lambda U = 0.
\end{align}%
Thus we have Eq.(\ref{eq:u}):
\begin{align}
U = XV (V^T V+ \lambda I)^{-1}. \label{eq:u}
\end{align}%

A2. Fixing $U$, solve $V$. Take derivative of Eq.(\ref{eq:j1}) with respect to $V$ and set it to zero,
\begin{align}
\frac{\partial J_1}{\partial V} = -X^T U + V U^T U + \lambda V = 0.
\end{align}%
Thus we can get the solution Eq.(\ref{eq:v}):
\begin{align}
V = X^T U (U^T U+ \lambda I)^{-1}. \label{eq:v}
\end{align}%

It is easy to prove that function value $J_1$ is monotonically decreasing. To minimize objective function of Eq.(\ref{eq:j1}), we propose an iterative Algorithm \ref{alg:rsvd}. We initialize $V$ using a random matrix. Then we minimize Eq.(\ref{eq:j1}) iteratively, until it converges. The converge speed is actually affected by the regularization weight parameter $\lambda$. In experiment section, we will show that RSVD converges faster than SVD ($\lambda=0$).

Will the random initialization of matrix $V$ in step 1 of Algorithm \ref{alg:rsvd} affect the final solution?
Is the solution of Algorithm \ref{alg:rsvd} unique?
Below, we present theoretical analysis and vigorously prove that
there is a unique global solution and the above iterative algorithm converge to the global solution.

\begin{algorithm}[t]
\small
\caption{Regularized SVD (RSVD)}
\label{alg:rsvd}
\begin{algorithmic}[1]
\Require Data matrix $X \in \Re^{n \times m}$, rank $k$, regularization weight parameter $\lambda$
\Ensure Factor matrices $U \in \Re^{n \times k}$, $V \in \Re^{m \times k}$
\State Initialize matrix $V$ using a random matrix
\Repeat
\State Compute $U$ using Eq.(\ref{eq:u})
\State Compute $V$ using Eq.(\ref{eq:v})
\Until $J_1$ converges
\end{algorithmic}
\end{algorithm}

\section{RSVD solution is in SVD subspace}
Here we establish two important theoretical results: Theorems \ref{tm:thm21} and \ref{tm:thm22}, which show RSVD solution is in SVD subspace.

The singular value decomposition (SVD) of $X$ is given as
\begin{align}
X = F \Sigma G^T, \label{eq:xsvd}
\end{align}%
where $F =(f_1,\cdots,f_r) \in \Re^{n \times r}$ are the left singular vectors,
$G =(g_1,\cdots,g_r) \in \Re^{m \times r}$ are the right singular vectors,
 $\Sigma = \mbox{diag}(\sigma_1, ..., \sigma_r) \in \Re^{r \times r}$ contains singular values,
 and $r$ is the rank of $X$. $\sigma_1, ..., \sigma_r$ are sorted in decreasing order.

We now present Theorem \ref{tm:thm21} and \ref{tm:thm22} to show that RSVD solution is in subspace of SVD solution.
Let $V$ be the optimal solution of RSVD. Let the QR decomposition of $V \in \Re^{m \times k}$ be
 \begin{align}
V = V_{\perp}\Omega, \label{eq:qr}
\end{align}%
where $V_{\perp}\in \Re^{m \times k}$ is an orthonormal matrix and $\Omega \in \Re^{k \times k}$ is an upper triangular matrix.
\begin{theorem}
\label{tm:thm21}
Matrix $\Omega$ in Eq.(\ref{eq:qr}) is a diagonal matrix.
\end{theorem}
\begin{proof}
Substituting Eq.(\ref{eq:u}) back into Eq.(\ref{eq:j1}), we have a formulation of $V$ only,
{\scriptsize
\begin{align}
J_1(V)= \Tr(X^TX - X^T XV (V^TV+\lambda I)^{-1} V^T  + \lambda V^TV )\label{eq:xv}
\end{align}
}%
Using Eq.(\ref{eq:qr}) and fixing $V_{\perp}$, we have
\begin{align}
J_1(\Omega)= \Tr(A - B \Omega (\Omega^T\Omega+\lambda I)^{-1} \Omega^T  + \lambda \Omega^T\Omega ), \label{eq:jo}
\end{align}
where $A=X^T X, B=V_{\perp}^T X^T XV_{\perp}$ are independent of $\Omega$. Let the eigen-decomposition of $\Omega^T\Omega = C \Lambda C^T, \Omega=\Lambda^{1/2}C^T$. Eq.(\ref{eq:jo}) now becomes
\begin{align}
J_1(\Lambda)= \Tr(A - B \Lambda^{1/2} (\Lambda+\lambda I)^{-1} \Lambda^{1/2}  + \lambda \Lambda ),
\end{align}
where $C$ cancel out exactly. Thus $J_1$ is independent of $C$;
$J_1$ depends on the eigenvalues of $\Omega^T\Omega$. For this reason, we can
 set $C=I$, $\Omega=\Lambda^{1/2}$ is a diagonal matrix.
\end{proof}

\begin{theorem}
\label{tm:thm22}
RSVD solution $V_{\perp}$ of Eq.(\ref{eq:qr}) is in the subspace of SVD singular vectors $G$, as in Eq.(\ref{eq:xsvd}).
\end{theorem}

\begin{proof}
Using Eq.(\ref{eq:qr}) and fixing $\Omega$, Eq.(\ref{eq:xv}) can be written as
\begin{align}
J_1(V_{\perp})= \Tr(A -V_{\perp}^T G \Sigma^2 G^T V_{\perp} D  + E ),
\end{align}
where $A = X^TX, D = \Omega (\Omega^T\Omega+\lambda I)^{-1} \Omega^T, E = \lambda \Omega^T\Omega$ is independent of $V_{\perp}$.

We now show that
\newline
(L1) For any $V_{\perp}$, $J_1(V_{\perp})$ has a lower bound $J_b$:
\begin{align}
J_1(V_{\perp}) \geq  J_b = \Tr(A -\Sigma^2  D  + E ),  \label{eq:lowb}
\end{align}
and
\newline
(L2) the optimal $V^*_{\perp} = G$.

To prove (L2), we see that when  $V^*_{\perp} = G$,
\begin{align}
J_1(V^*_{\perp}) = \Tr(A -   G^T G \Sigma^2 G^T G D \Sigma^2  D  + E ) = J_b,
\end{align}
i.e., $J_1(V_{\perp})$ reaches the lowest possible value, the global minima. Thus $V^*_{\perp} = G$ is the global optimal solution.

To prove (L1) we use Von Neumann's trace inequality, which states that for any two matrices $P, Q$, with diagonal singular value matrix $\Lambda_P$ and $\Lambda_Q$ respectively, $|\Tr(PQ)| \leq \Tr(\Lambda_P \Lambda_Q)$.
In our case, $Q=D= \Omega (\Omega^T\Omega+\lambda I)^{-1} \Omega^T$
is already a non-negative diagonal matrix. $P =  V_{\perp}^T G \Sigma^2 G^T V_{\perp}$, and $P$'s singular values are $\Sigma^2 >0$. Thus we have
\begin{align}
\Tr (V_{\perp}^T G \Sigma^2 G^T V_{\perp} D) \leq  \Tr (\Sigma^2 D). \label{eq:von}
\end{align}
Adding constant matrices $A,E$ and notice the negative sign, the inequality Eq.(\ref{eq:von}) gives the lower bound Eq.(\ref{eq:lowb}). This proves (L1).
\end{proof}

\section{Closed form solution of RSVD}

The key results of this paper is that although RSVD
is non-convex, we can obtain the global optimal solution, as below.

Using Theorems \ref{tm:thm21} and \ref{tm:thm22}, we now present the closed form solution of RSVD.
Given Eq.(\ref{eq:u}) and Eq.(\ref{eq:qr}), as long as we solve $\Omega$, we can get the closed form solution of RSVD $U$ and $V$. The closed form solution is presented in Theorem \ref{tm:thm3}.
\begin{theorem}
\label{tm:thm3}
Let SVD of the input data $X$ be $ X = F\Sigma G^T$ as in Eq.(\ref{eq:xsvd}).
Let
$( U^* ,V^*)$ be the global optimal solution of RSVD.
We have
\begin{align}
U^* = F_k ,
V^* = G_k\Omega
\end{align}%
where
$F_k = (f_1,\cdots,f_k)$, $G_k= (g_1,\cdots,g_k)$, and
%
$\Omega = \mbox{diag}(\omega_1, ..., \omega_k) \in \Re^{k \times k}$,
\begin{align}
\omega_i = \sqrt{(\sigma_i - \lambda)_+}, \label{eq:wi}, \; i=1,\cdots,k
\end{align}%
%
\end{theorem}

\begin{proof}
Substituting Eq.(\ref{eq:qr}) back to Eq.(\ref{eq:xv}) and using $G^TG = I$, we have
\begin{align}
J_1(\Omega) = \Tr( A - \Sigma^2\Omega^2(\Omega^2 + \lambda I)^{-1} + \lambda \Omega^2),
\end{align}%
where $A = G \Sigma^2G^T$ is a constant independent of $\Omega$. Noting that all the matrices are diagonal, we can minimize $J_1$ element-wisely with respect to $\omega_i$, $i = 1, ..., k$. Taking the derivative of $J_1$ respect to $\omega_i$ and setting it to zero, we have
\begin{align}\label{EQ:closeForm}
\omega_i^2 = (\sigma_i - \lambda)_+ ,
\end{align}%
because $\omega_i \geq 0$. From this, we finally have Eq.(\ref{eq:wi}).
\end{proof}%

One consequence of Theorem \ref{tm:thm3} is that the choice of
parameter $\lambda$ become obvious:
it should be closedly related to parameter $k$, the rank of $U, V$.

We should set $ \lambda$ such that  $\{\omega_i\} > 0$ so that
no columns of $U, V$ are waisted.

Another point to make is that directly computing $U,V$ from Algorithm 1
is generally faster than compute the SVD of $X$, because generally,
$k$ are much smaller than rank($X$), thus computing full rank SVD of $X$ is not
necessary.

{\bf Computational complexity analysis}. From Theorem 3, a single SVD computation can obtain the global solution.
If we desire a strong regularization, we set $\lambda$ large, and compute SVD upto the appropriate rank using Eq.(\ref{EQ:closeForm}).
The computation complexity is $O[k(n+m)\min(n,m)]$.
We may use Algorithm 1 to directly compute RSVD without computing SVD.
Theoretically, this is
  faster than computing the SVD because the regularization term  $(V^T V+ \lambda I)^{-1}$ makes
  Algorithm \ref{alg:rsvd} converge faster for larger regularization $\lambda$. The
  The computation complexity is $O(kmn)$. Inverting
  the $k \times k$ matrix  $(V^T V+ \lambda I)$ is fast since $k$ is typically much smaller than $\min(n,m)$.


Numerical experiments are given below.

\section{Application to Recommender Systems}

Recommender system generally uses collaborative filtering
~\cite{billsus1998learning}. This is often viewed
  as a dimensionality reduction problem and their best-performing algorithm is based on singular value decomposition (SVD) of a user ratings matrix. By exploiting the latent structure (low rank) of user ratings, SVD approach eliminates the need for users to rate common items. In recent years, SVD approach has been widely used as an efficient collaborative filtering algorithm \cite{jester} \cite{billsus1998learning} \cite{sarwar2001item} \cite{kurucz2007methods} \cite{sarwar2000application}.

User-item rating matrix $X$ generally is a very sparse matrix with only values 1,2,3,4,5. Zeros elements imply that matrix entry
has not been filled
because each user usually only rates a few items. Similarly, each item is only rated by a small subset of users.
Thus recommender system is in essence of estimating missing values of the rating matrix.

Assume we have a user-item rating matrix $X \in \Re^{n \times m}$, where $n$ is the number of users and $m$ is the number of items (i.g., movies).
Some ratings in matrix $X$ are missing. Let  $\Omega$ be the set of
$i,j$ indexes that the matrix element has been set.
Recommender system using SVD solves the following problem:
\begin{align}
\min_{U,V} \| X - UV^T \|^2_{\Omega}, \label{eq:svdobj}
\end{align}%
with fixed rank $k$ of $U,V$,
where for any matrix $A$, $\| A\|^2_{\Omega} = \sum_{(i,j) \in \Omega} A_{ij}^2$.

Low-rank $U$ and $V$ can expose the underlying latent structure.
However, because $X$ is sparse, $U,V$ is forced to match
a sparse structure and thus could overfit.
Adding a regularization term will make $U$ and $V$ more smooth, and thus could reduce the overfitting.
For this reason, we propose the regularized SVD recommender system as the following problem
\begin{align}
\min_{U,V} \| X - UV^T \|^2_{\Omega} + \lambda \|U\|_F^2 + \lambda \|V\|_F^2. \label{eq:rsvdobj}
\end{align}%
Both Eqs.(\ref{eq:svdobj},\ref{eq:rsvdobj}) are solved by an EM-like algorithm ~\cite{srebro2003weighted} ~\cite{koren2009matrix},
 which first fills the missing values with column or row averages,
solving the low-rank reconstruction problem as the usual problem without missing values, and then update the missing values of $X$ using the new SVD result. This is repeated until convergence. The RSVD algorithm presented above is used to solve Eqs.(\ref{eq:svdobj},\ref{eq:rsvdobj}).

\section{Experiments}

Here we compare recommender systems using the Regularized SVD of Eq.(\ref{eq:rsvdobj}) and classical SVD of Eq.(\ref{eq:svdobj}) on four datasets.

{\noindent\bf Datasets}.
Table \ref{tab:data} summarizes the user number $n$ and item number $m$ of the 4 datasets.
\begin{table}[t!]
\centering
\small
\caption{Recommender system datasets.}
\label{tab:data}
\scalebox{1}{
\begin{tabular}{c|cc}
\hline\hline
Data    & user ($n$)   & item ($m$) \\
\hline
MovieLens    &943    & 1682   \\
RottenTomatoes &931 & 1274    \\
Jester1    &1731 & 100   \\
Jester2  & 1706 & 100\\
\hline\hline
\end{tabular}}
\end{table}

\begin{figure}[t!]
\centering
\begin{subfigure}{.24\textwidth}
  \centering
  \includegraphics[width=.96\textwidth]{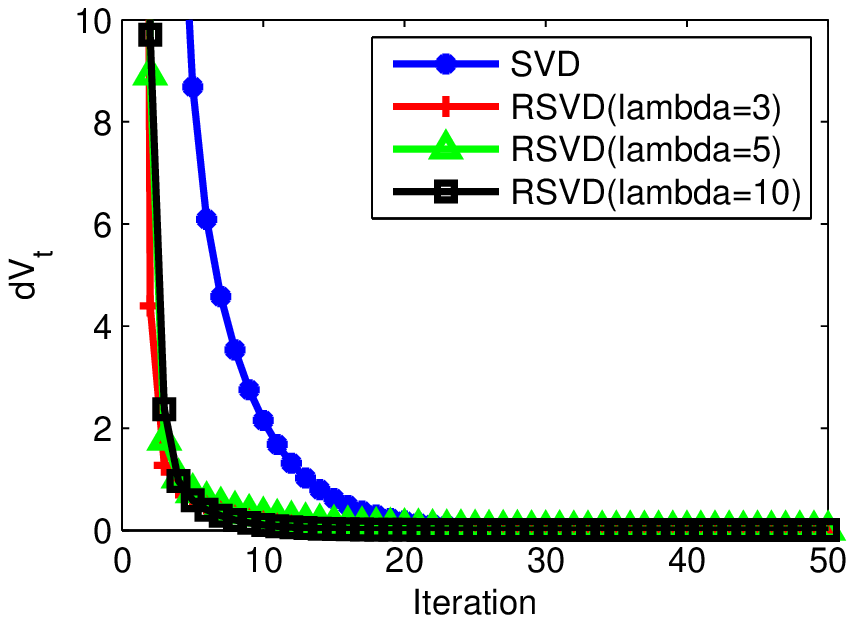}
  \caption{MovieLens($k=3$)}
  \label{fig:movielens_j1}
\end{subfigure}%
\begin{subfigure}{.24\textwidth}
  \centering
  \includegraphics[width=.96\textwidth]{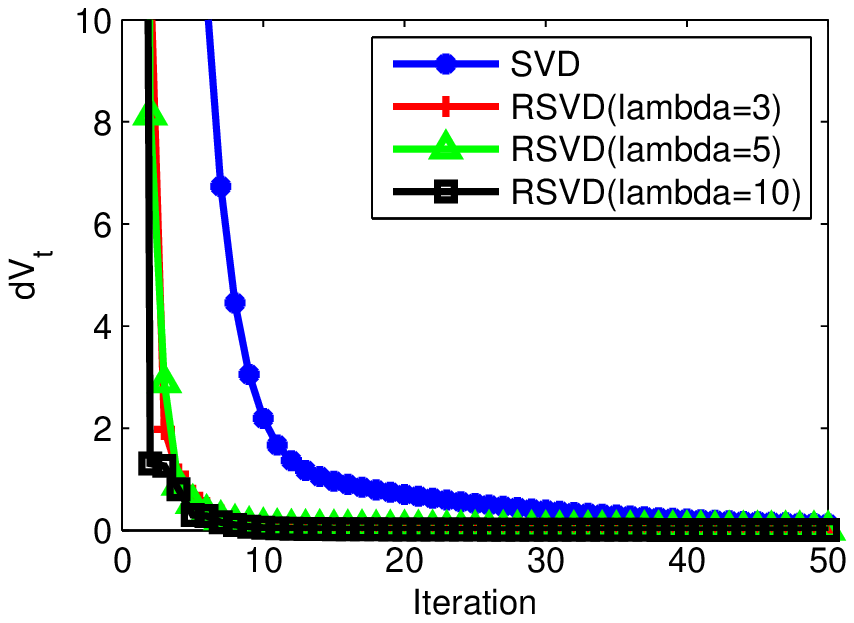}
  \caption{RottenTomatoes($k=3$)}
  \label{fig:rotentomato_j1}
\end{subfigure}
\begin{subfigure}{.24\textwidth}
  \centering
  \includegraphics[width=.96\textwidth]{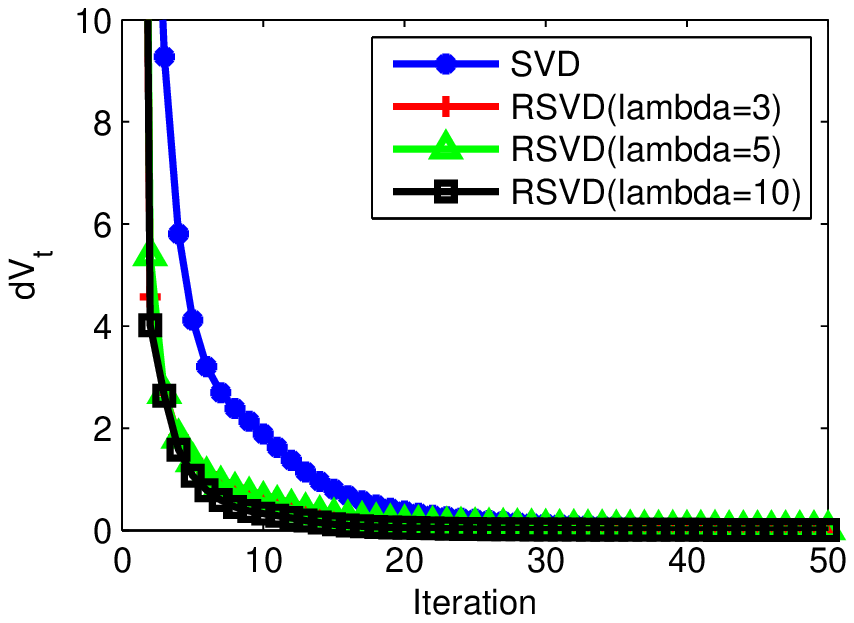}
  \caption{Jester1($k=14$)}
  \label{fig:js1_j1}
\end{subfigure}%
\begin{subfigure}{.24\textwidth}
  \centering
  \includegraphics[width=.96\textwidth]{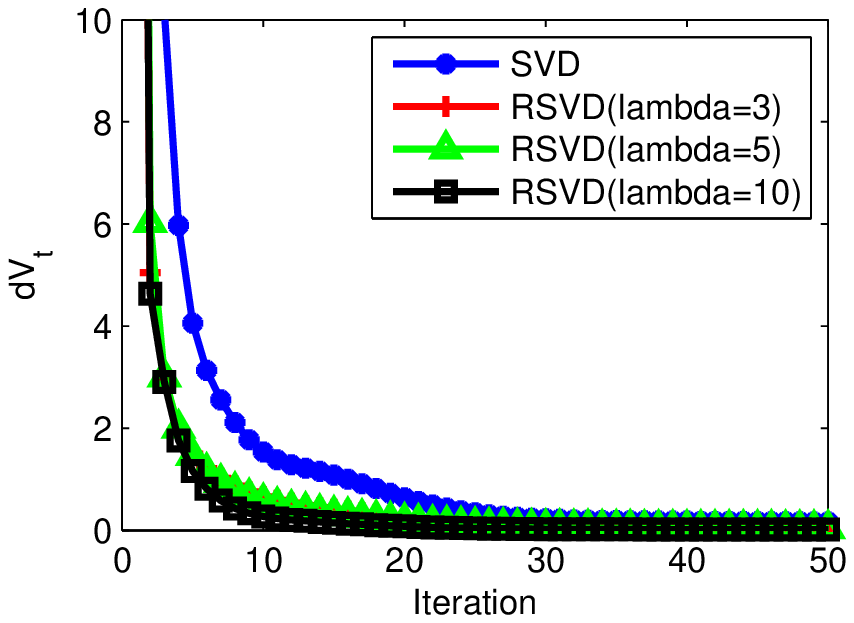}
  \caption{Jester2($k=14$)}
  \label{fig:js2_j1}
\end{subfigure}
\caption{RSVD convergence speed comparison at different $\lambda$, see Eq.(\ref{eq:dV_t}).}
\label{fig:j1}
\end{figure}

\begin{figure}[t!]
\centering
\begin{subfigure}{.24\textwidth}
  \centering
  \includegraphics[width=.96\textwidth]{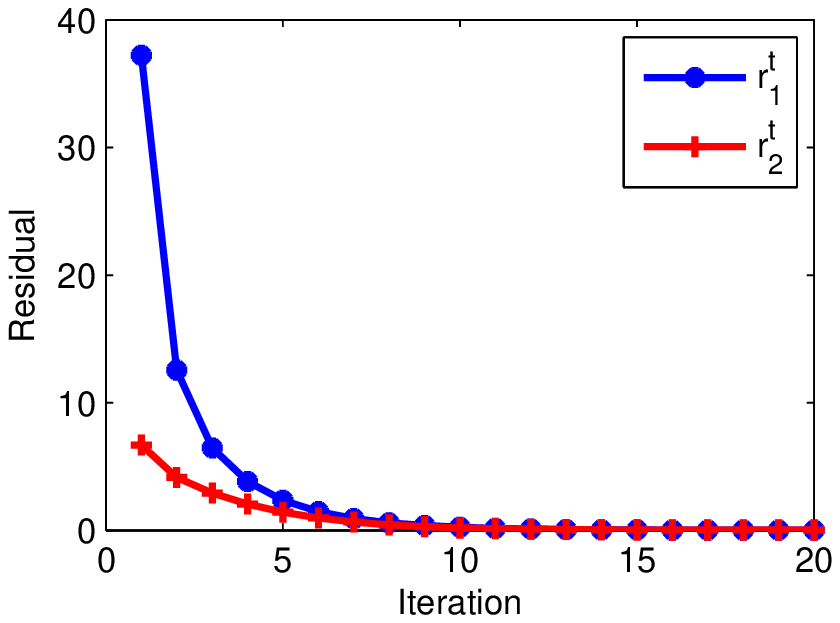}
  \caption{MovieLens($k=3, \lambda=3$)}
  \label{fig:share_basis_movielens}
\end{subfigure}%
\begin{subfigure}{.25\textwidth}
  \centering
  \includegraphics[width=.96\textwidth]{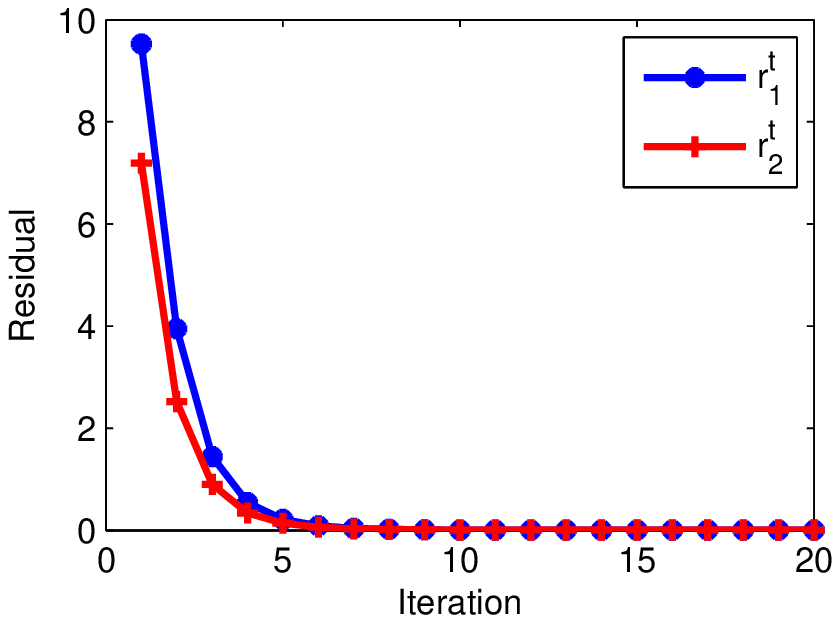}
  \caption{RottenTomatoes($k=3, \lambda=3$)}
  \label{fig:share_basis_roten}
\end{subfigure}
\begin{subfigure}{.24\textwidth}
  \centering
  \includegraphics[width=.96\textwidth]{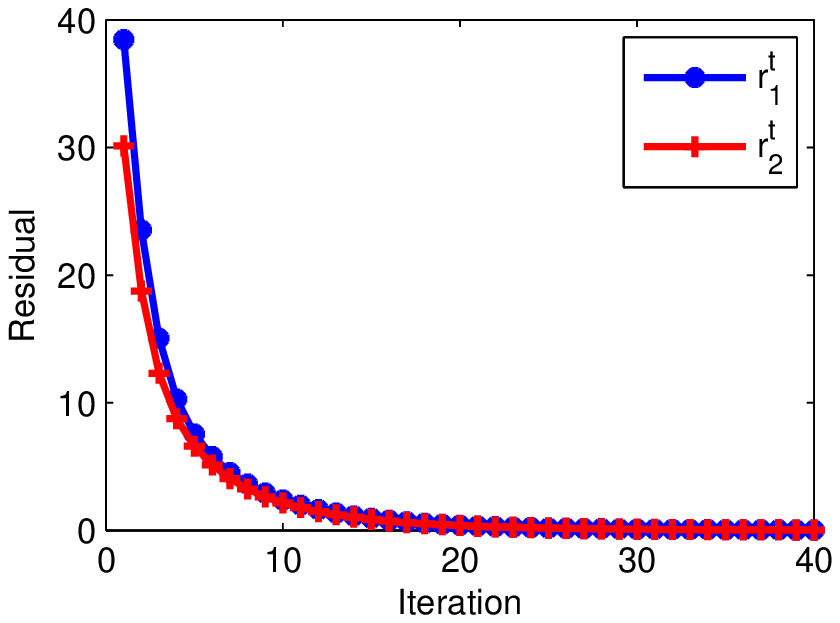}
  \caption{Jester1($k=14, \lambda=3$)}
  \label{fig:share_basis_js1}
\end{subfigure}%
\begin{subfigure}{.24\textwidth}
  \centering
  \includegraphics[width=.96\textwidth]{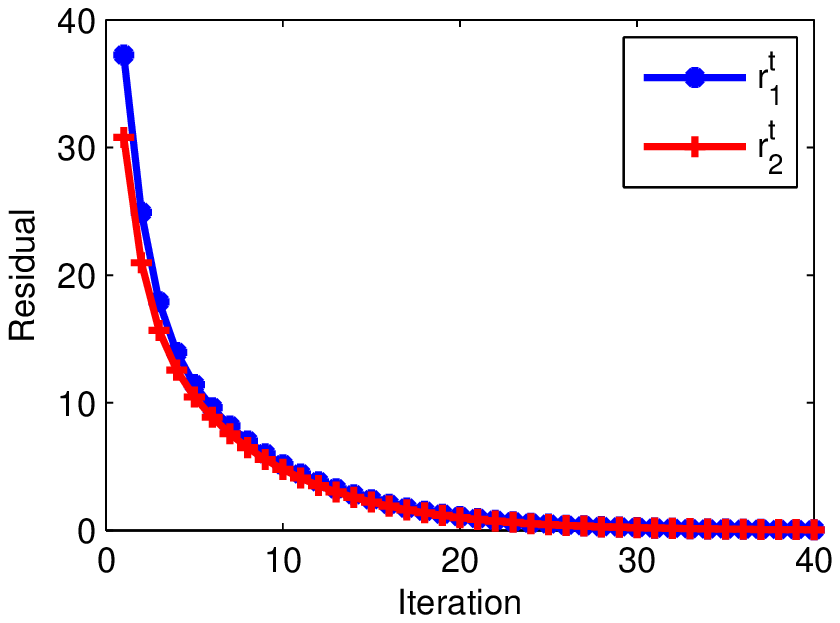}
  \caption{Jester2($k=14, \lambda=3$)}
  \label{fig:share_basis_js2}
\end{subfigure}
\caption{RSVD share the same SVD subspace, see Eqs.(\ref{eq:r1},\ref{eq:r2}).}
\label{fig:shared_basis}
\end{figure}

\begin{figure}[t!]
\centering
\begin{subfigure}{.24\textwidth}
  \centering
  \includegraphics[width=.96\textwidth]{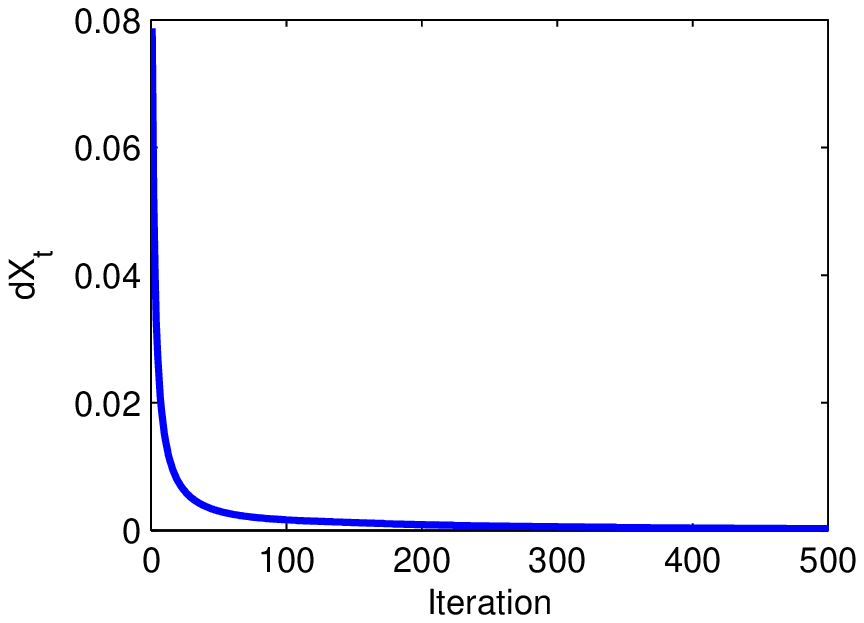}
  \caption{MovieLens($k=3, \lambda=3$)}
  \label{fig:movielens_em}
\end{subfigure}%
\begin{subfigure}{.25\textwidth}
  \centering
  \includegraphics[width=.96\textwidth]{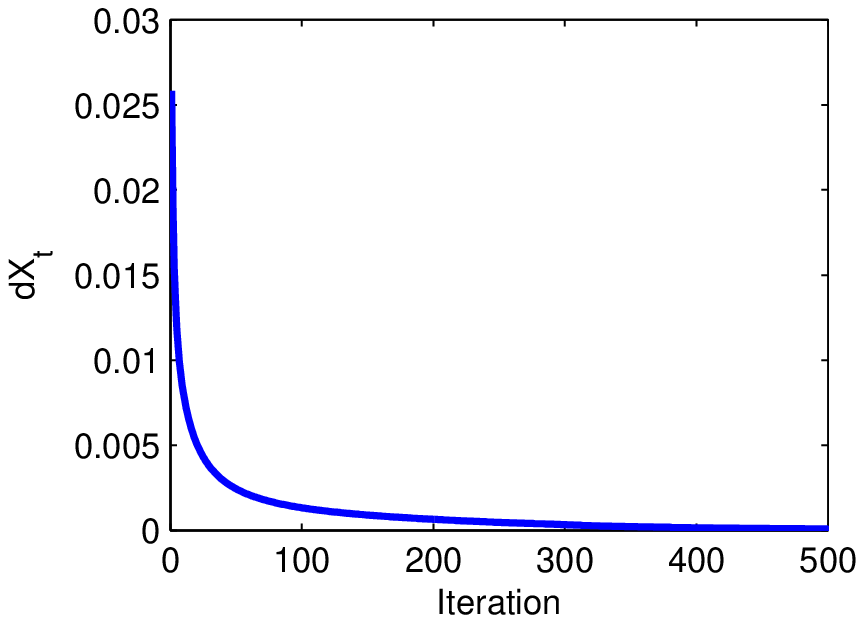}
  \caption{RottenTomatoes($k=3, \lambda=3$)}
  \label{fig:rotentomato_em}
\end{subfigure}
\begin{subfigure}{.24\textwidth}
  \centering
  \includegraphics[width=.96\textwidth]{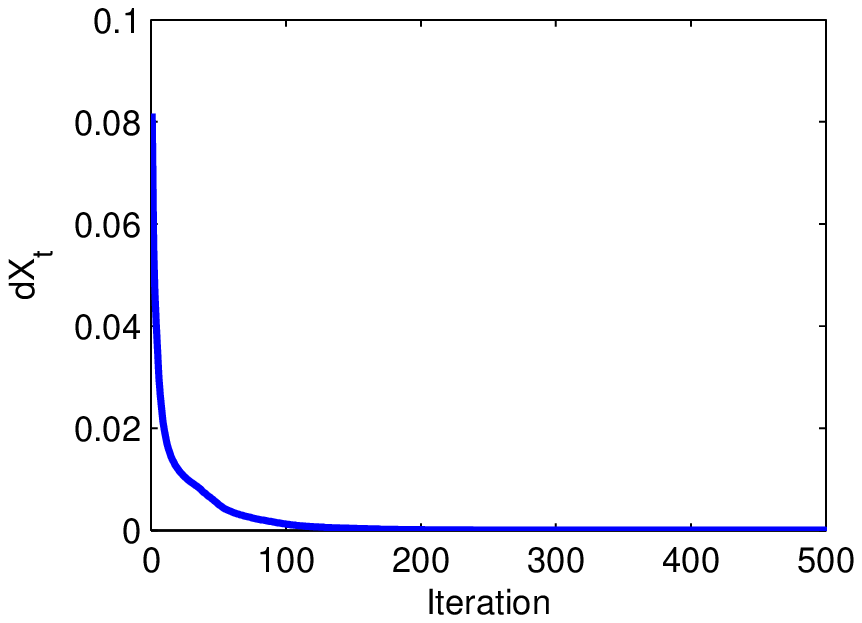}
  \caption{Jester1($k=14, \lambda=3$)}
  \label{fig:js1_em}
\end{subfigure}%
\begin{subfigure}{.24\textwidth}
  \centering
  \includegraphics[width=.96\textwidth]{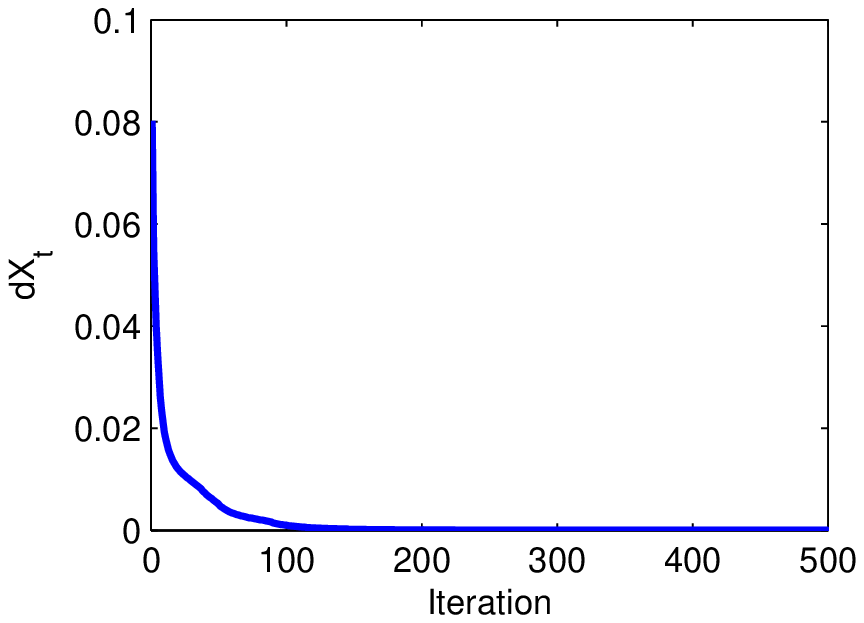}
  \caption{Jester2($k=14, \lambda=3$)}
  \label{fig:js2_em}
\end{subfigure}
\caption{Convergence of the solution to Recommender Systems of Eqs.(\ref{eq:svdobj},\ref{eq:rsvdobj}) as the iteration of EM steps.}
\label{fig:em}
\end{figure}

\begin{figure}[t!]
\centering
\begin{subfigure}{.24\textwidth}
  \centering
  \includegraphics[width=.96\textwidth]{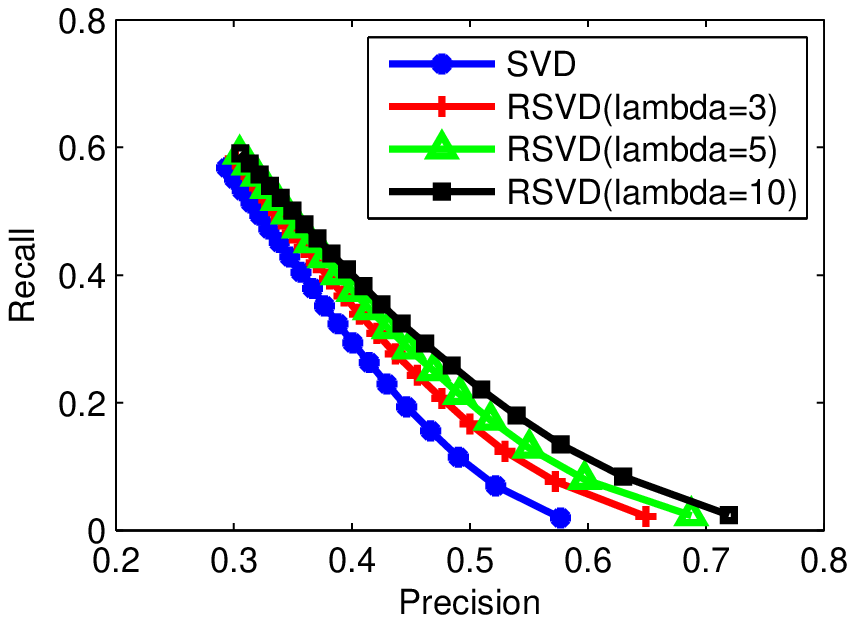}
  \caption{k=3}
  \label{fig:movielens_k3}
\end{subfigure}%
\begin{subfigure}{.24\textwidth}
  \centering
  \includegraphics[width=.96\textwidth]{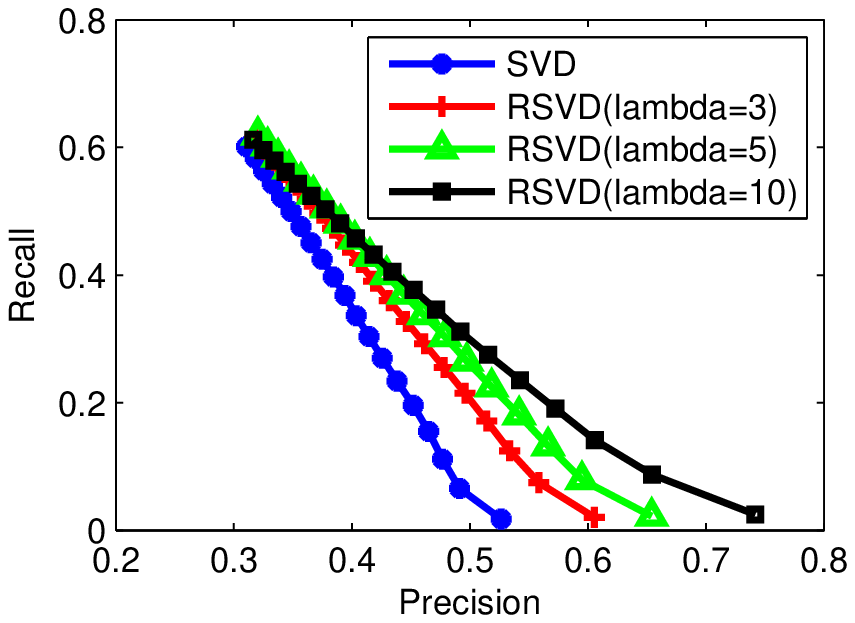}
  \caption{k=5}
  \label{fig:movielens_k5}
\end{subfigure}
\begin{subfigure}{.24\textwidth}
  \centering
  \includegraphics[width=.96\textwidth]{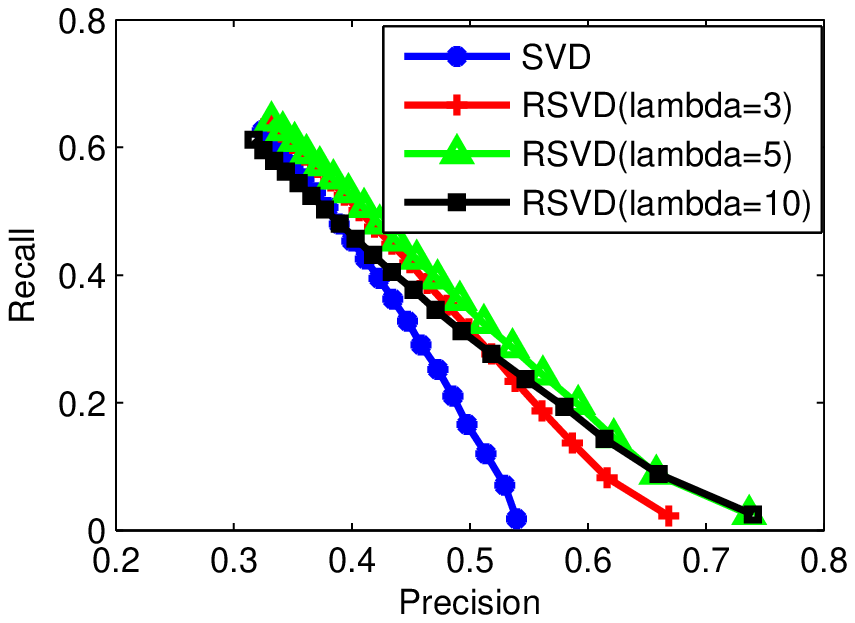}
  \caption{k=7}
  \label{fig:movielens_k7}
\end{subfigure}%
\begin{subfigure}{.24\textwidth}
  \centering
  \includegraphics[width=.96\textwidth]{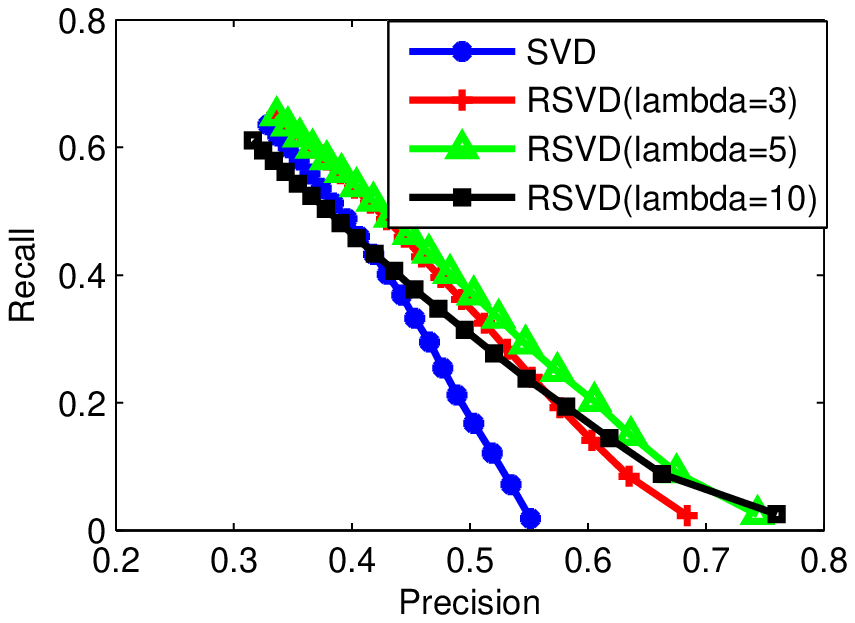}
  \caption{k=9}
  \label{fig:movielens_k9}
\end{subfigure}
\caption{Precision and Recall curves on MovieLens.}
\label{fig:movielens}
\end{figure}

\begin{figure}[t!]
\centering
\begin{subfigure}{.24\textwidth}
  \centering
  \includegraphics[width=.96\textwidth]{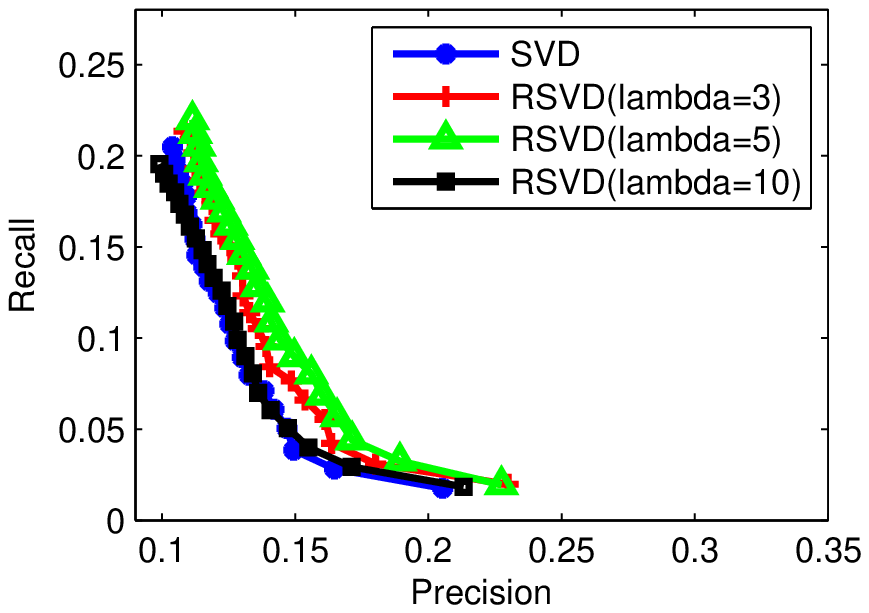}
  \caption{k=3}
  \label{fig:rotentomato_k3}
\end{subfigure}%
\begin{subfigure}{.24\textwidth}
  \centering
  \includegraphics[width=.96\textwidth]{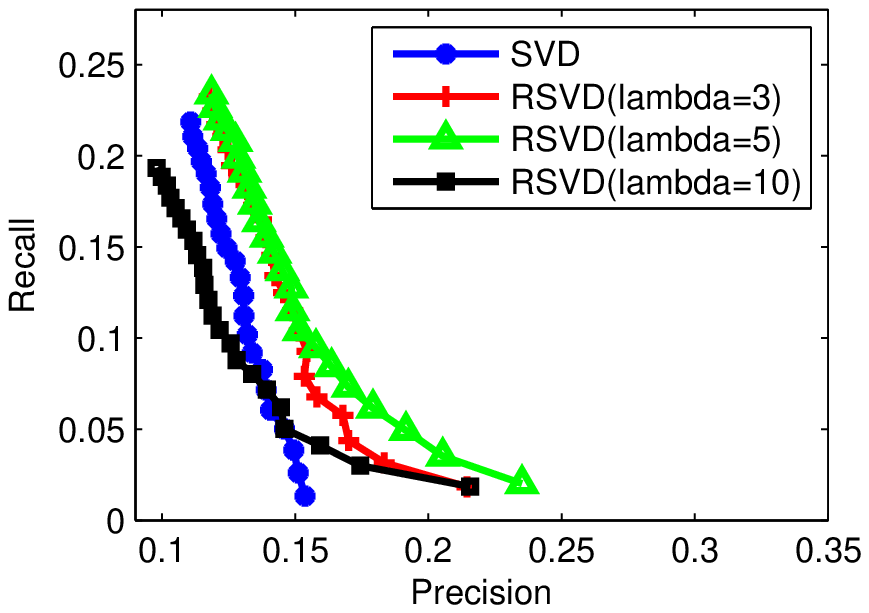}
  \caption{k=5}
  \label{fig:rotentomato_k5}
\end{subfigure}
\begin{subfigure}{.24\textwidth}
  \centering
  \includegraphics[width=.96\textwidth]{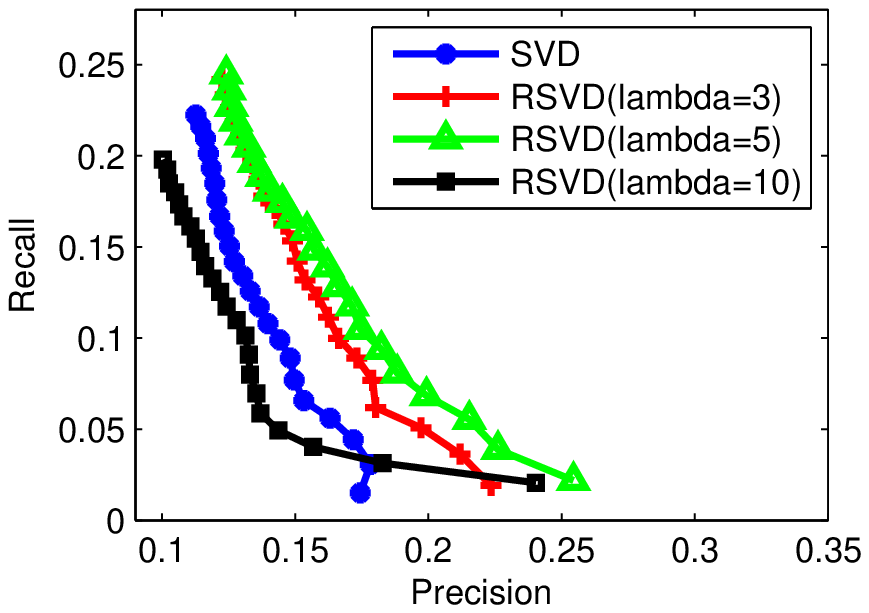}
  \caption{k=7}
  \label{fig:rotentomato_k7}
\end{subfigure}%
\begin{subfigure}{.24\textwidth}
  \centering
  \includegraphics[width=.96\textwidth]{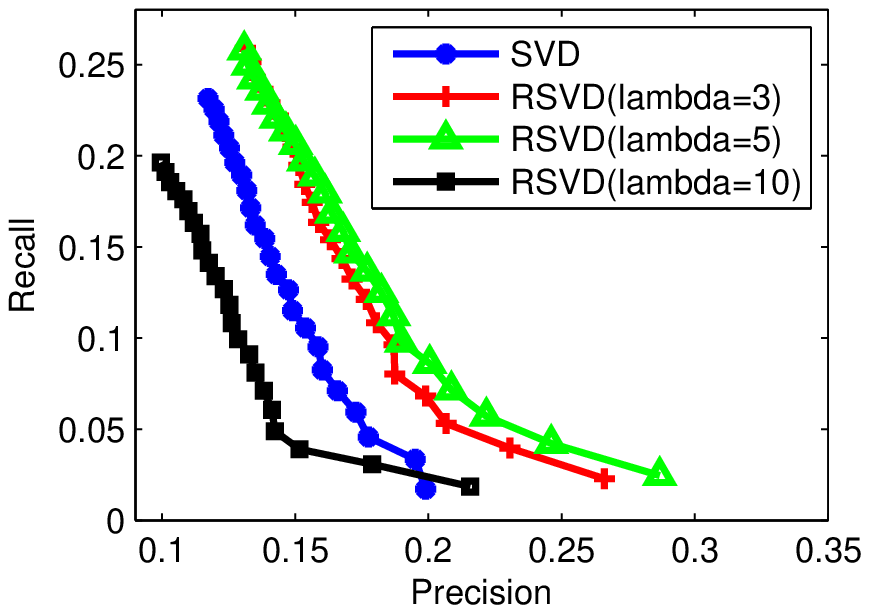}
  \caption{k=9}
  \label{fig:rotentomato_k9}
\end{subfigure}
\caption{Precision and Recall curves on RottenTomatoes.}
\label{fig:rotentomato}
\end{figure}

\begin{figure}[t!]
\centering
\begin{subfigure}{.24\textwidth}
  \centering
  \includegraphics[width=.96\textwidth]{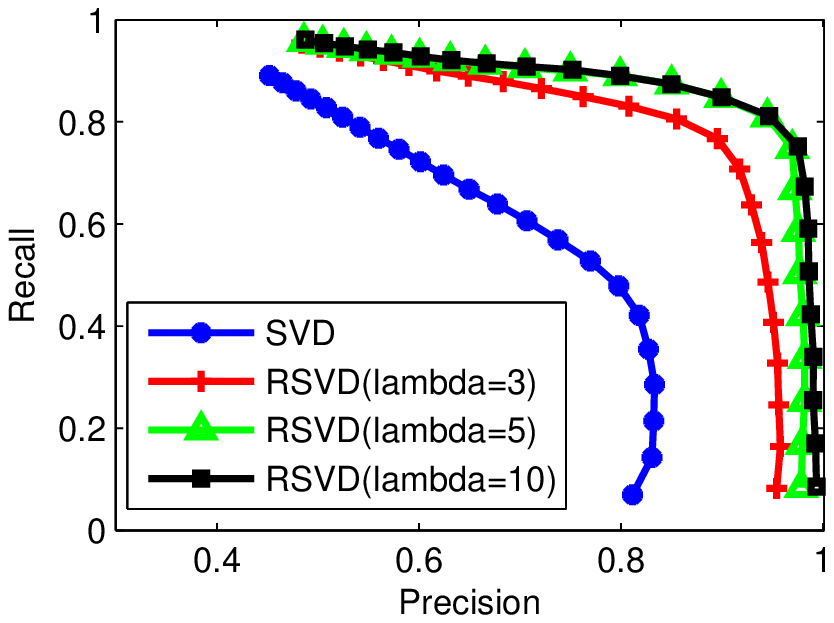}
  \caption{k=14}
  \label{fig:js1_k14}
\end{subfigure}%
\begin{subfigure}{.24\textwidth}
  \centering
  \includegraphics[width=.96\textwidth]{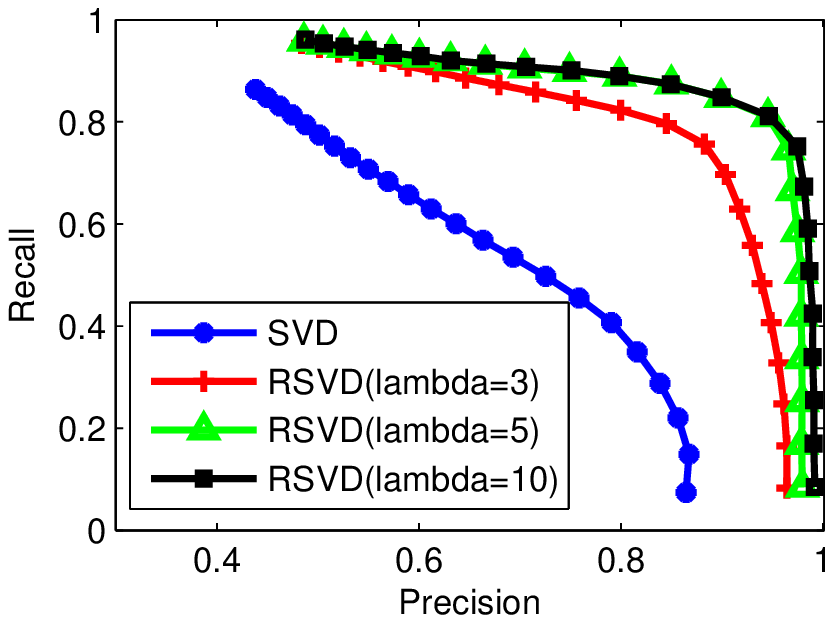}
  \caption{k=16}
  \label{fig:js1_k16}
\end{subfigure}
\begin{subfigure}{.24\textwidth}
  \centering
  \includegraphics[width=.96\textwidth]{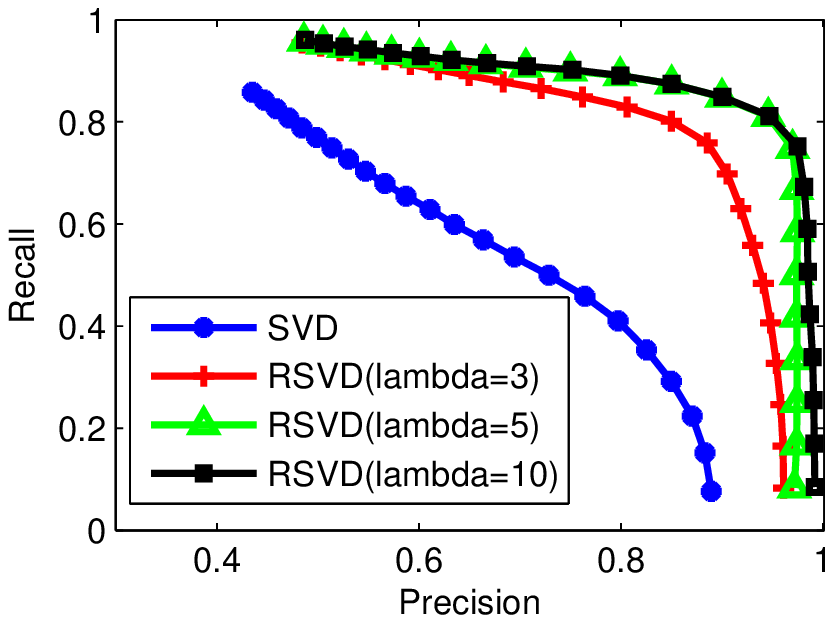}
  \caption{k=18}
  \label{fig:js1_k18}
\end{subfigure}%
\begin{subfigure}{.24\textwidth}
  \centering
  \includegraphics[width=.96\textwidth]{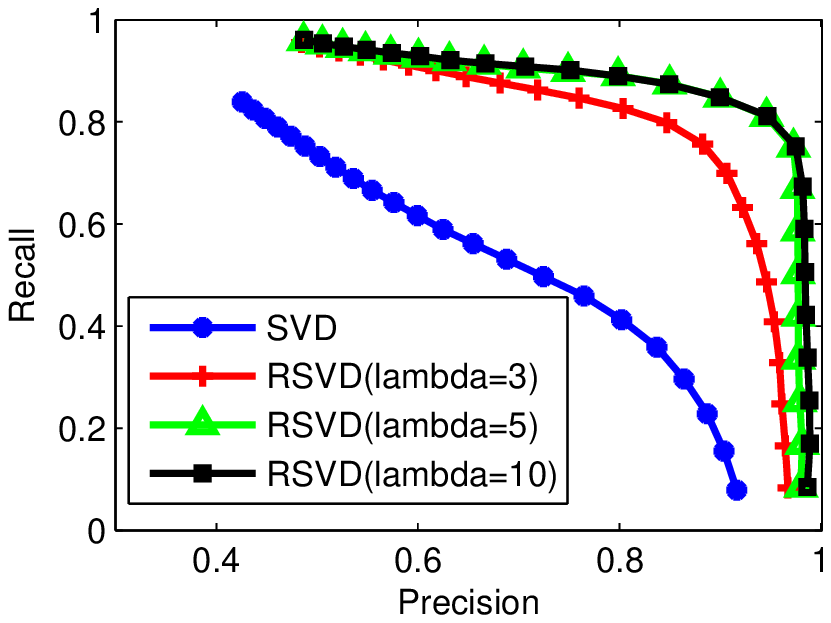}
  \caption{k=20}
  \label{fig:js1_k20}
\end{subfigure}
\caption{Precision and Recall curves on Jester1.}
\label{fig:js1}
\end{figure}

\begin{figure}[t!]
\centering
\begin{subfigure}{.24\textwidth}
  \centering
  \includegraphics[width=.96\textwidth]{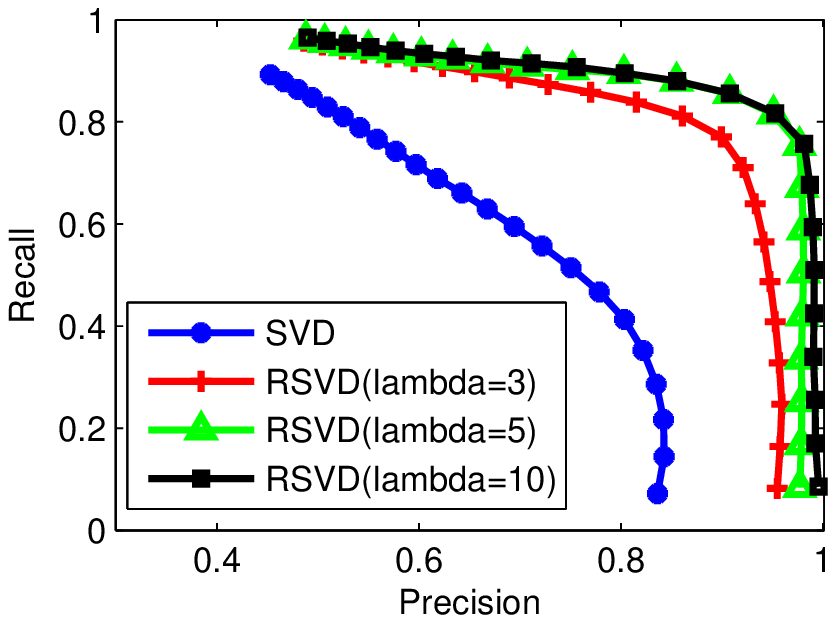}
  \caption{k=14}
  \label{fig:js2_k14}
\end{subfigure}%
\begin{subfigure}{.24\textwidth}
  \centering
  \includegraphics[width=.96\textwidth]{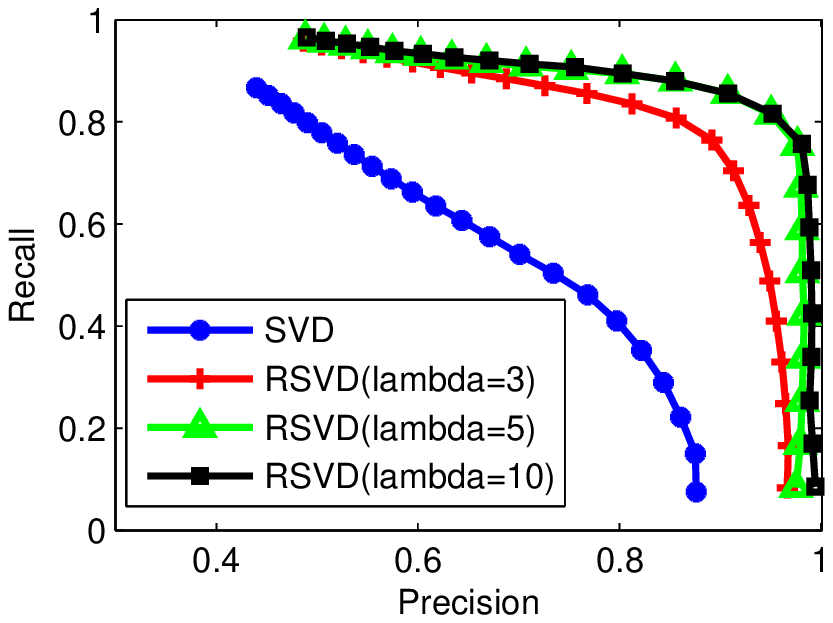}
  \caption{k=16}
  \label{fig:js2_k16}
\end{subfigure}
\begin{subfigure}{.24\textwidth}
  \centering
  \includegraphics[width=.96\textwidth]{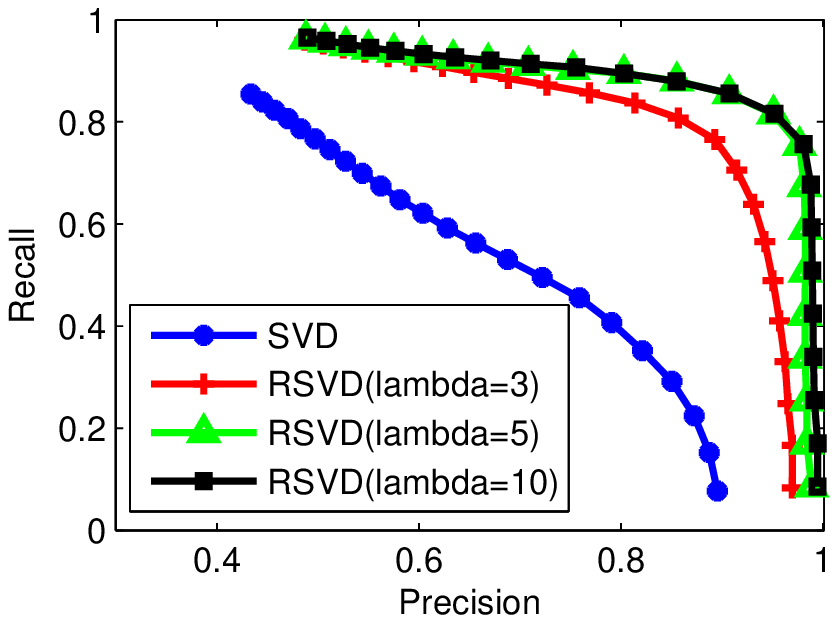}
  \caption{k=18}
  \label{fig:js2_k18}
\end{subfigure}%
\begin{subfigure}{.24\textwidth}
  \centering
  \includegraphics[width=.96\textwidth]{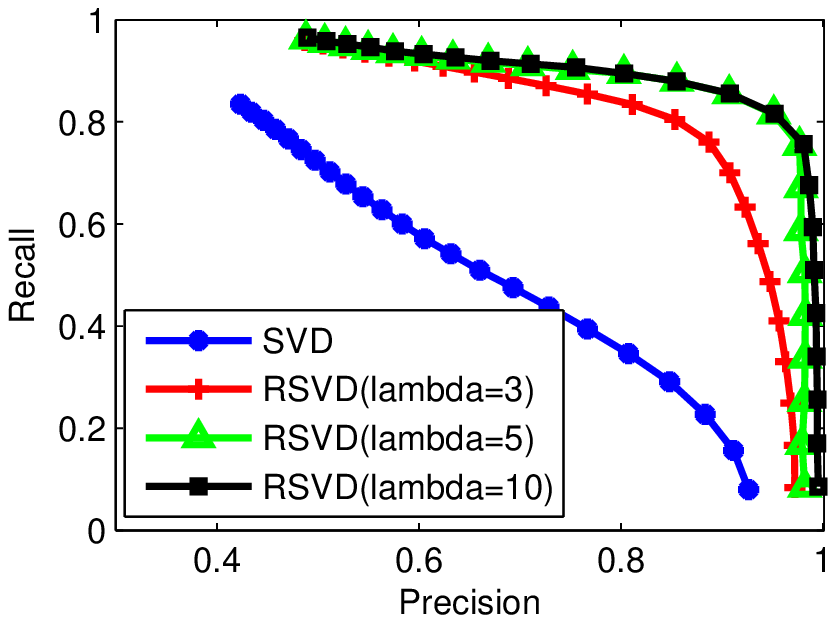}
  \caption{k=20}
  \label{fig:js2_k20}
\end{subfigure}
\caption{Precision and Recall curves on Jester2.}
\label{fig:js2}
\end{figure}


{\bf MovieLens} \cite{movielens1} \cite{movielens2} 
This data set consists of 100,000 ratings from 943 users on 1,682 movies. Each user has at least 20 ratings and the average number of ratings per user is 106.


{\bf RottenTomatoes} \cite{movielens1} \cite{imdb} \cite{rotten}
This dataset contains 931 users and 1,274 artists. Each user has at least 2 movie ratings and the average number of ratings per user is 17.

{\bf Jester1} \cite{jester} Jester is an online Joke recommender system and it has 3 .zip files.
Jester1 dataset contains 24,983 users and is the 1st .zip file of Jester data.
In our experiments, we choose 1,731 users with each user having 40 or less joke ratings. The average number of ratings per user is 37.

{\bf Jester2} \cite{jester} 
Jester2 dataset contains 23,500 users and is the 2nd .zip file of Jester data.
In our experiments, we choose 1,706 users with each user having 40 or less joke ratings. The average number of ratings per user is 37.

\subsection{Training data}

Following standard approach, we convert all rated entries to 1 and all missing value entries remains zero.
 The evaluation methodology is: (1) construct training data by converting some 1s in the rating matrix into 0s, which is called ``mask-out", (2) check if recommender algorithms can correctly recommend these masked-out ratings. Suppose we are given a set of user-item rating records, namely $X \in \Re^{n \times m}$, where $X$ is the rating matrix, $n$ is user number and $m$ is item number. Each row of $X$ denotes one user. To evaluate the performance of a recommender system algorithm, we need to know how accurate this algorithm can predict those $1$s. We refer to the original data matrix as ground truth and mask out some ratings for some selected users. The {\bf mask-out} process is as follows:
\begin{enumerate}
  \item Find training users: those users with more than $t$ ratings are selected as training users, where $t$ is a threshold and $t$ is a number related to the average ratings per user ($m_{rating}$). $t$ controls the number of training users ($n_{user}$).
  \item Mask out training ratings: for $n_{user}$ selected training users, select $n_{mask}$ ratings randomly per training user. In the user-item matrix $X$, we change those $1$s into $0$s.
\end{enumerate}%
Table \ref{tab:train} shows the training data mask-out settings used in our experiments.
It should be noted that these parameters are only one setting of constructing training datasets. Different settings will not make much difference, as long as we compare different recommender system algorithms on the same training dataset.

\begin{table}[t!]
\centering
\small
\caption{Training data parameter settings.}
\label{tab:train}
\scalebox{1}{
\begin{tabular}{c|cccc}
\hline\hline
Data    & $t$  &  $n_{user}$ & $m_{rating}$ & $n_{mask}$ \\
\hline
MovieLens      &100 & 361 &106  & 90 \\
RottenTomatoes &40  & 86  &17   & 35\\
Jester1        &37  & 803  &37  & 35\\
Jester2  	   &37  & 774  &37  & 35\\
\hline\hline
\end{tabular}}
\end{table}

\subsection{Top-N recommendation evaluation}

To check if recommender algorithms can correctly recommend these masked-out ratings, we use Top-N recommendation evaluation method. Top-N recommendation is an algorithm to identify a set of $N$ items that will be of interest to a certain users \cite{karypis2001evaluation}  \cite{deshpande2004item} \cite{sarwar2000application}. We use three metrics widely used in information retrieval community: recall, precision and $F_1$ measure. For each user, we first define three sets: $\mathbb{M}$, $\mathbb{T}$ and $\mathbb{H}$:

{\bf $\mathbb{M}$}: Mask-out set. Size is $n_{mask}$. This set contains the ratings that are masked out(those values in data matrix $X$ were changed from $1$ to $0$).

{\bf $\mathbb{T}$}: Top-N set. Size is $N$. This set contains the $N$ ratings that has the highest values (score) after using recommendation algorithm.

{\bf $\mathbb{H}$}: Hit set. This set contains the ratings that appear both in $\mathbb{M}$ set and $\mathbb{T}$ set, $\mathbb{H} = \mathbb{M} \cap \mathbb{T}$.

\noindent Recall and precision are then defined as follows:
\begin{align}
\label{eq:pr}
\mbox{Recall} = \frac{ \mbox{size of set } \mathbb{H} }{\mbox{size of set } \mathbb{M}}, 
\mbox{Precision} = \frac{ \mbox{size of set } \mathbb{H}}{\mbox{size of set } \mathbb{T}}
\end{align}%
$F_1$ measure \cite{yang1999re} combines recall and precision with an equal weight in the following form:
\begin{align}
\label{eq:f1}
F_1 = \frac{ 2 \times \mbox{Recall} \times \mbox{Precision}}{\mbox{Recall} + \mbox{Precision}}
\end{align}%
We will get a pair of recall and precision using each $N$. In experiments, we use $N$ from 1 to $2 n_{mask}$, where $n_{mask}$ is the number of ratings masked out per user. Thus, we can get a precision-recall curve in this way.

\subsection{RSVD convergence speed comparison}

Convergence speed is important for a faster iterative algorithm. We will compare the convergence speed of RSVD with iterative SVD algorithm ($\lambda=0$). We define residual $dV_t$ to measure the difference of $V_t$ and $V_{t-1}$ in two consecutive iterations:
\begin{align}
dV_t = \| V_{t} - V_{t-1} \|_F,  \label{eq:dV_t}
\end{align}%
where $t$ is the iteration number of Algorithm \ref{alg:rsvd}. We compare RSVD with SVD ($\lambda=0$) using different regularization weight parameter $\lambda=3, 5, 10$. Figure \ref{fig:j1} shows the $dV_t$ decreases quickly along with iterations and RSVD converges faster than SVD.

\subsection{RSVD share the same SVD subspace}

From Theorem \ref{tm:thm22}, we know that the solution of RSVD should be in the subspace of SVD solution. Formally, let $U_t, V_t$ be the solution of RSVD after $t$ iterations, $F,G$ be the solution of SVD, $X=FG^T$.
We now introduce Eq.(\ref{eq:r1}) and Eq.(\ref{eq:r2}) to measure the difference between $U_t, V_t$ and $F,G$. $r_1^t$ and $r_2^t$ are defined as
\begin{align}
r_1^t = \| U_t- FA_t\|_F^2, \label{eq:r1}\\
r_2^t = \| V_t- GB_t\|_F^2. \label{eq:r2}
\end{align}%
In order to minimize $r_1^t$ and $r_2^t$, the solution of $A_t$ and $B_t$ can be given as:
\begin{align}
A_t = (F^T F)^{-1} F^T U_t, \label{eq:au} \\
B_t = (G^T G)^{-1} G^T V_t. \label{eq:av}
\end{align}%
Substituting Eq.(\ref{eq:au}) and Eq.(\ref{eq:av}) back to Eq.(\ref{eq:r1}) and Eq.(\ref{eq:r2}), we get the minimized residual $r_1^t$ and $r_2^t$.
If $r_1^t$ and $r_2^t$ are equal to $0$, it means that RSVD solution $U_t$ and $V_t$ share the same subspace as SVD solution $F$ and $G$. Figure \ref{fig:shared_basis} shows residual $r_2^t$ and $r_3^t$ converges to 0 after a few iterations.

\subsection{Convergence of recommender system solution}

Solutions to the recommender systems Eqs.(\ref{eq:svdobj},\ref{eq:rsvdobj}) converge. The EM-like algorithm has been shown effective in solving recommender systems \cite{srebro2003weighted} \cite{koren2009matrix} \cite{kurucz2007methods} . 
We show the solution $(X_t)_{\Omega}$ converges after $t$ iterations of EM-like iterations by
 using the difference,
\begin{align}
\label{eq:dX_t}
dX_t = \frac{1}{\sqrt{ N_\Omega} } \|X_{t} - X_{t-1} \|_{\Omega}.
\end{align}
where $N_{\Omega}$ is size of set $\Omega$.
Figure \ref{fig:em} shows the experiment result of $dX_{t}$.
As we can see, for all the 4 datasets, the solution converges in about 100 to 200 iterations. 

\subsection{Precision-Recall Curve}

In this part, we compare the precision and recall of RSVD and SVD using different rank $k$ and regularization weight parameter $\lambda$. We use these $k$ and $\lambda$ settings because both RSVD and SVD models with these settings produce the best precision and recall.
All the curves are the average results of 5 random run.

Figure \ref{fig:movielens} shows MovieLens data using SVD and RSVD with rank $k=3,5,7,9$. For each rank $k$, we compare SVD and RSVD with regularization weight parameter $\lambda=3,5,10$. As we can see, for each rank $k$, RSVD performs better than SVD generally.
Choosing $\lambda$ properly could improve SVD algorithm and achieve the best precision and recall results.

Figure \ref{fig:rotentomato} shows RottenTomatoes data using SVD and RSVD with rank $k=3,5,7,9$. In all figures, RSVD with $\lambda = 5$ performs the best.

Figure \ref{fig:js1} shows Jester1 data using SVD and RSVD with rank $k=14,16,18,20$. It is very easy to find that RSVD with $\lambda=5,10$ produce the best precision result for this data.

Figure \ref{fig:js2} shows Jester2 data using SVD and RSVD with rank $k=14,16,18,20$. We can see from the results that RSVD with $\lambda=5,10$ produce the best precision result.
As in Jester1 data, RSVD algorithm improves SVD significantly.

\subsection{$F_1$ measure}
$F_1$ measure combines precision and recall at the same time and can be used a good metric. $F_1$ measure is defined in Eq.(\ref{eq:f1}). Since each $N$ gives a pair of precision and recall, we use $F_1$ measure when $N=n_{mask}$ as the standard. Because $N=n_{mask}$, if all the masked-out ratings are predicted correctly, the size of set $\mathbb{H}$ can be exactly $n_{mask}$, which means recall is 1. $F_1$ measure ranges from 0 to 1. A higher $F_1$ measure (close to 1) means that an algorithm has better performance.

Table \ref{tab:f1} shows the $F_1$ measure of the four datasets. Each row denotes a dataset with a specific rank $k$. The best $F_1$ measure is denoted in bold. As we can see, for all the datasets and ranks that we experimented, $\lambda=5$ is a good setting that produces the highest $F_1$ measure.
In all, RSVD performs much better than SVD in terms of $F_1$ measure. In applications, we can test different $\lambda$ and rank $k$ setting to find the best setting for specific problems.

\begin{table}[t]
\centering \small
\caption{$F_1$ measure (best values are in bold.).}
\label{tab:f1}
\scalebox{0.9}{
\begin{tabular}{c|cccc}
\hline\hline
Data    & SVD  &  RSVD & RSVD & RSVD \\
    &    &  ($\lambda=3$)& ($\lambda=5$) & ($\lambda=10$) \\
\hline
MovieLens (k=3)& 0.3700&	0.3850&	0.3922&	{\bf 0.4005}\\
MovieLens (k=5)& 0.3875&	0.4100&	0.4199&	{\bf 0.4232}\\
MovieLens (k=7)& 0.4152&	0.4391&	{\bf 0.4439}&	0.4231\\
MovieLens (k=9)& 0.4220&	0.4497&	{\bf 0.4542}&	0.4244\\
\hline
RottenTomatoes (k=3)& 0.1220&	0.1308&	{\bf 0.1337}&	0.1235\\
RottenTomatoes (k=5)&0.1302&	0.1413&	{\bf 0.1436}&	0.1176\\
RottenTomatoes (k=7)&0.1315&	0.1501&	{\bf 0.1543}&	0.1228\\
RottenTomatoes (k=9)&0.1422&	0.1614&	{\bf 0.1651}&	0.1240\\
\hline
Jester1 (k=14)& 0.6587&	0.8241&	{\bf 0.8668}&	0.8665\\
Jester1 (k=16)& 0.6201&	0.8151&	{\bf 0.8667}&	0.8659\\
Jester1 (k=18)& 0.6188&	0.8213&	{\bf 0.8672}&	0.8666\\
Jester1 (k=20)& 0.6077&	0.8177&	{\bf0.8667}&	0.8658\\
\hline
Jester2 (k=14) &0.6506&	0.8305&	{\bf 0.8730}&	{\bf 0.8730}\\
Jester2 (k=16) &0.6261&	0.8277&	{\bf 0.8732}&	0.8725\\
Jester2 (k=18) &0.6114&	0.8288&	{\bf 0.8729}&	0.8723\\
Jester2 (k=20) &0.5908&	0.8259&	0.8721&	{\bf 0.8722}\\
\hline\hline
\end{tabular}}
\end{table}

\section{Conclusion}

In conclusion, SVD is the mathematical basis of principal component analysis (PCA). We present a regularized SVD (RSVD), present an efficient computational algorithm, and provide several theoretical analysis. We show that although RSVD is non-convex, it has a closed-form global optimal solution.
Finally, we apply regularized SVD to the application of recommender system and experimental results show that regularized SVD (RSVD) outperforms SVD significantly.


\bibliographystyle{abbrv}
\bibliography{bib1}

\end{document}